%% file: main.tex
\pdfoutput=1
\documentclass[11pt]{article}

\usepackage{acl2023}

\usepackage{times}
\usepackage{latexsym}

\usepackage[T1]{fontenc}

\usepackage[utf8]{inputenc}

\usepackage{microtype}
\usepackage{mathtools}
\usepackage{subcaption}

\usepackage{booktabs}
\usepackage{algorithm}
\usepackage{algorithmic}

\usepackage[space]{grffile}
\usepackage{courier}
\usepackage{amsthm}
\usepackage{xcolor}
\usepackage{listings}
\usepackage{epstopdf}
\usepackage{epsfig}
\usepackage{amsmath}
\usepackage{multirow}
\usepackage{bm}
\usepackage{adjustbox}
\usepackage{enumitem}
\usepackage{amssymb}
\usepackage{caption}

\usepackage[english]{babel}
\usepackage{fancyhdr}
\usepackage{longtable}
\usepackage{tikz}

\DeclareMathOperator{\E}{\mathop{\mathbb{E}}}

\newtheorem{theorem}{Theorem}

\title{Free Lunch for Efficient Textual Commonsense Integration in Language Models}

\author{Wanyun Cui \\
  Shanghai University of Finance and Economics \\
  \texttt{cui.wanyun@shufe.edu.cn} \\ \And
  Xingran Chen \\
  University of Michigan \\
  \texttt{chenxran@umich.edu} \\}

\begin{document}

\maketitle

\begin{abstract}

%Integrating knowledge into language models has significantly advanced the state-of-the-art for a variety of tasks. Previous research has focused on integrating symbolic knowledge (e.g., knowledge graphs), with knowledge atoms represented by embedding matrices. 

Recent years have witnessed the emergence of textual commonsense knowledge bases, aimed at providing more nuanced and context-rich knowledge. The integration of external commonsense into language models has been shown to be a key enabler in advancing the state-of-the-art for a wide range of NLP tasks. However, incorporating textual commonsense descriptions is computationally expensive, as compared to encoding conventional symbolic knowledge. In this paper, we propose a method to improve its efficiency without modifying the model. We group training samples with similar commonsense descriptions into a single batch, thus reusing the encoded description across multiple samples. One key observation is that the upper bound of batch partitioning can be reduced to the classic {\it graph k-cut problem}. Consequently, we propose a spectral clustering-based algorithm to solve this problem. Extensive experiments illustrate that the proposed batch partitioning approach effectively reduces the computational cost while preserving performance. The efficiency improvement is more pronounced on larger datasets and on devices with more memory capacity, attesting to its practical utility for large-scale applications.

\end{abstract}

\input{fast_intro}
%\input{fast_model}
\input{fast_problem}

\input{fast_method}

\input{fast_related}

\input{fast_exp}
\input{fast_conclu}

\bibliography{fast}
\bibliographystyle{acl_natbib}

\appendix

% \section{Example Appendix}
% \label{sec:appendix}

% This is an appendix.

\end{document}

%% file: fast_intro.tex
\section{Introduction}

While pre-trained language models have made substantial progress in natural language processing, they still lack certain knowledge. Thus it is critical to incorporate external knowledge sources~\cite{peters2019knowledge, zhang2019ernie, logan2019barack}. Previous research has primarily focused on incorporating symbolic knowledge from structured knowledge graphs. 
Recently, realizing the lack of expressiveness and contextualization of symbolic knowledge, many forms of commonsense knowledge bases are constructed, such as if-then knowledge~\cite{sap2019atomic} and discourse knowledge~\cite{fang2021discos}. The integration of such textual commonsense knowledge into language models has been shown to improve the state of the art for various tasks, such as named entity recognition~\cite{wu2020scalable} and commonsense knowledge base completion~\cite{malaviya2020commonsense}.

%However, symbolic knowledge graphs have several limitations, including a lack of expressiveness and contextualization. Specifically, many forms of commonsense knowledge are stored in text form, such as if-then knowledge~\cite{sap2019atomic} and discourse knowledge~\cite{fang2021discos}. The integration of such textual commonsense knowledge into language models has been shown to improve the state-of-the-art for various tasks, such as named entity recognition~\cite{wu2020scalable} and commonsense knowledge base completion~\cite{malaviya2020commonsense}. With the recent development of textual commonsense knowledge bases~\cite{sap2019atomic,hwang2020comet,fang2021discos}, it is reasonable to expect that language models can effectively and efficiently integrate such knowledge bases in order to improve their general text understanding capabilities.

However, integrating such commonsense knowledge are computationally expensive. Commonsense knowledge in text form requires more complex encoders (e.g. Transformer~\cite{vaswani2017attention}), as opposed to the simple lookup operation for discrete symbolic knowledge. The feed-forward and back-propagation process for the text encoder is significantly more computationally expensive than the standalone symbolic knowledge embeddings. Therefore, it is essential to reduce the computational cost for efficient integration of textual commonsense knowledge, particularly for large-scale applications.

In this paper, we propose a method to accelerate the process of incorporating textual commonsense knowledge into language models. Our approach is based on the observation that if multiple training samples in a mini-batch share the same commonsense description, the encoding for that description can be reused across those samples. In other words, we only need to encode each {\it distinct} description in a mini-batch once. For example, consider the training samples $x_{1\cdots 4}$ and the associated commonsense $t_{1\cdots 4}$ in Fig.~\ref{fig:idea}. In the batch partitioning in Fig.~\ref{fig:no_batch}, the samples in one batch have no shared descriptions, requiring seven times of commonsense encoding for $t_i$. However, in the batch partitioning shown in Fig.~\ref{fig:with_batch}, each description will be encoded only once, resulting in only four times of encoding for $t_i$. The cost of encoding the commonsense is significantly reduced by effective partitioning of the training samples. Therefore, our goal is to group the training samples in such a way as to minimize the total number of distinct commonsense descriptions per mini-batch.

\begin{figure}[tb]
\begin{subfigure}[b]{0.225\textwidth}
	\centering
		\includegraphics[scale=.4]{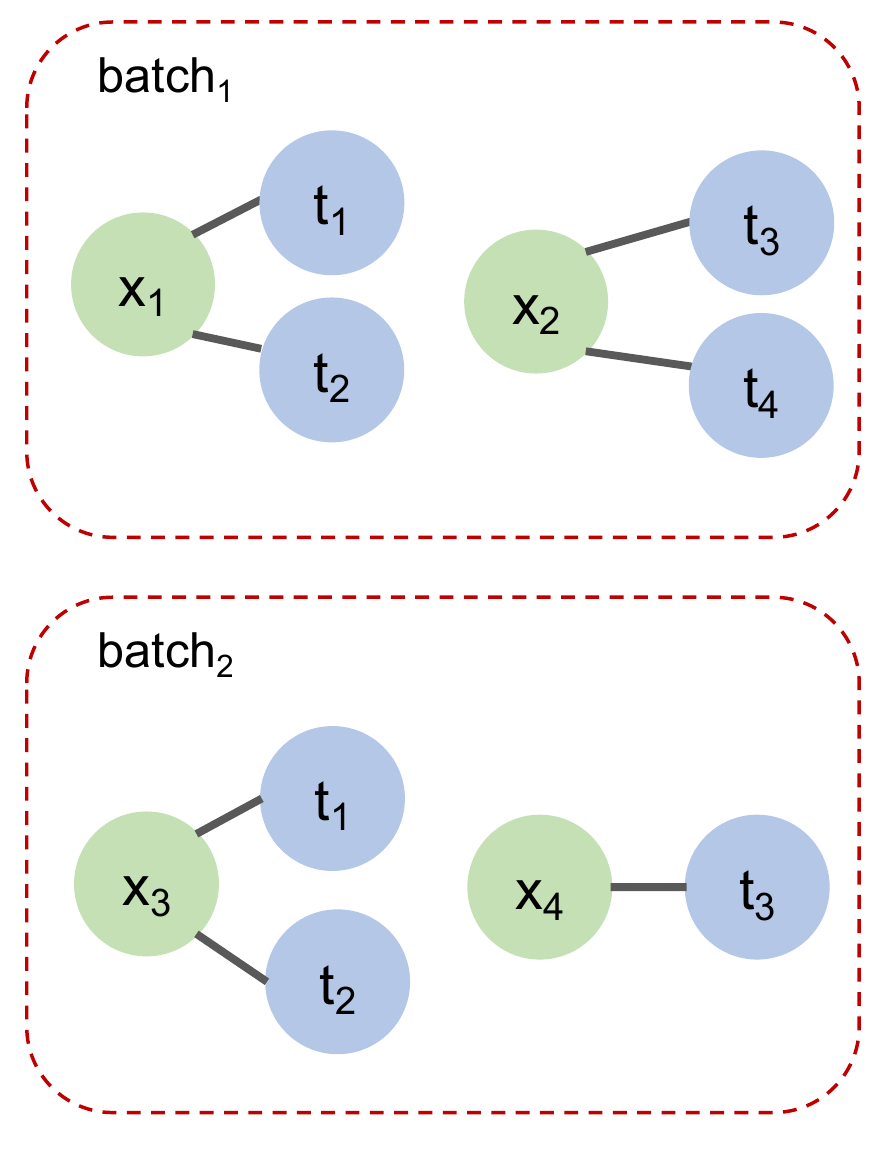}
% \vspace{0.1cm}
\caption{If samples are divided randomly into batches, a total of {\bf 7} times of encoding for $t_i$ is required.}
\label{fig:no_batch}
\end{subfigure}
\hspace{0.2cm}
\begin{subfigure}[b]{0.225\textwidth}
	\centering
		\includegraphics[scale=.4]{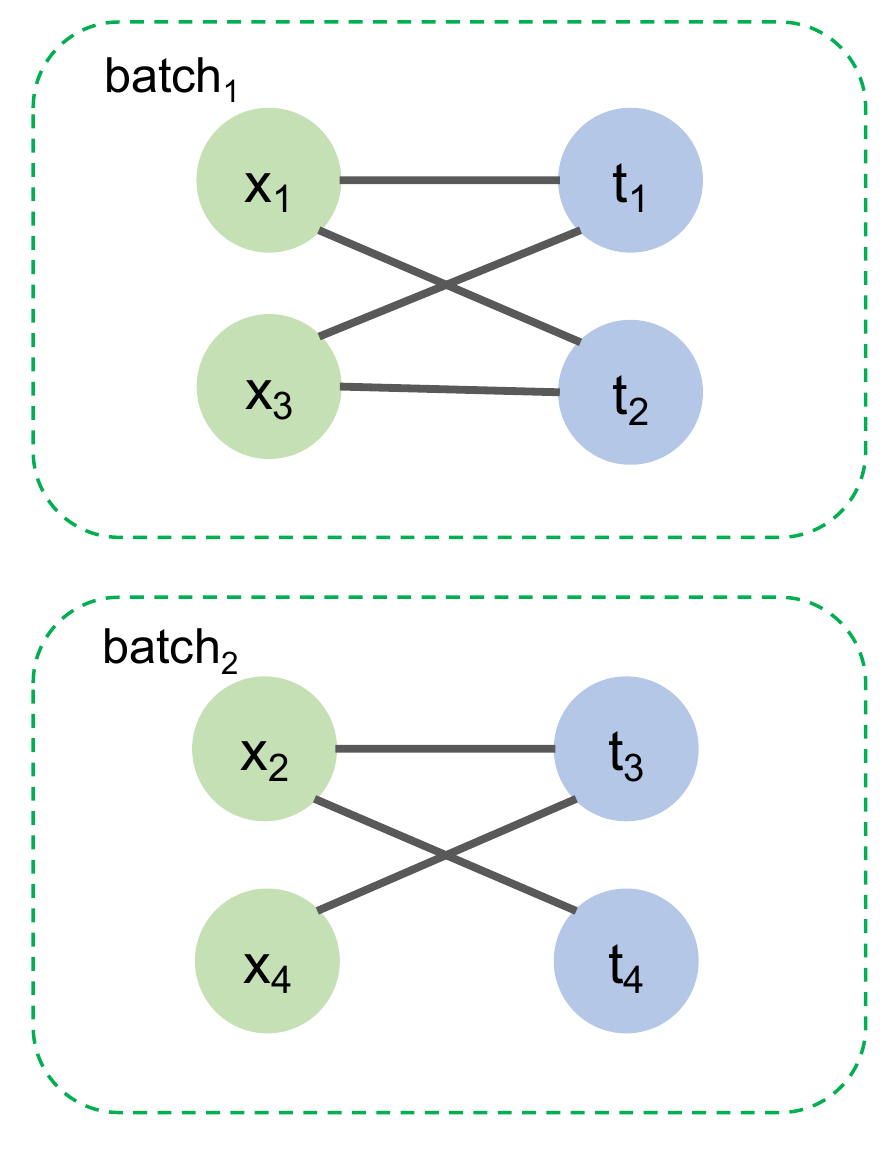}
\caption{If samples are divided delicately, only a total of {\bf 4} times of encoding for $t_i$ is required.}
\label{fig:with_batch}
\end{subfigure}
\caption{Idea of batch partitioning.}
\label{fig:idea}
\end{figure}

To optimize the batch partitioning, we begin by theoretically analyzing the objective (\S\ref{sec:complexity}). Our key observation is that the upper bound of the cost can be reduced to the well-studied {\it graph k-cut} problem~\cite{rangapuram2014tight} (\S~\ref{sec:upper}~\S~\ref{sec:graph_k_cut}). As a result, we minimize the upper bound instead by adapting the classic spectral clustering algorithm (\S~\ref{sec:spectral}). 
The average distinct commonsense descriptions per batch are approximated by the distance to the cluster centroid, and is optimized by spectral clustering. This is also empirically verified (\S~\ref{sec:exp:rationale}).

The main contributions of this paper are as follows: (1) We propose the use of batch partitioning for improving the efficiency of textual commonsense integration for language models. (2) We theoretically demonstrate that the batch partitioning problem can be reduced to the classic graph $k$-cut problem, and we use the well-studied spectral clustering to optimize it. (3) We empirically show that the efficiency of integrating commonsense descriptions can be significantly improved without sacrificing effectiveness. The acceleration is even more pronounced for large-scale training.

%In order to optimize the batch partitioning, we first mathematically analyze the objective. We show that the upper bound of this objective can be reduced to the classical {\it balanced graph k-cut} problem~\cite{rangapuram2014tight}. Therefore, we could minimize its upper bound instead. We adapt the classic spectral clustering algorithm~\cite{ng2002spectral} to solve this problem.

%Therefore, the key to integrating knowledge from a textual knowledge base is the efficiently and effectively represent the textual knowledge, rather than the knowledge retrieval in dense textual knowledge retriever.

%approach misses a large amount of information beyond the top $k$ candidate texts due to the accuracy issue of the retriever~\cite{humeau2019poly}. On the other hand, when using a textual knowledge base as the knowledge source, the manually labeled ligh-quality trigger words can be used as retriever. But the number of retrieved candidate textual knowledge still remains large (e.g. $>100$ in our experiments). This likewise causes a large amount of additional feedforward and backpropagation time.

%% file: fast_problem.tex
\section{The Batch Partitioning Problem}
In this section, we analyze the training efficiency w.r.t. batch partitioning. We first show in \S~\ref{sec:complexity} that the complexity of the model depends on the number of corresponding knowledge descriptions per sample. Then, in \S~\ref{sec:pdef}, we formally define this batch partitioning problem. 

\subsection{Model Setup and Complexity Analysis}
\label{sec:complexity}
%Taking OK-Transformer as the backbone, we analyze the complexity of encoding commonsense knowledge and propose the problem. OK-Transformer is a recent proposed model to introduce commonsense knowledge in language models. In order to introduce additional knowledge, earlier approaches required language models to be pre-trained on a large corpus along with external knowledge, which was time-consuming~\cite{peters2019knowledge,zhang2019ernie}. The recently proposed model, OK-Transformer~\cite{cui-chen-2022-enhancing}, was a breakthrough in its capacity to directly incorporate extra knowledge without pre-training. It exploits the knowledge token and attention mechanism to effectively integrate textual knowledge. The proposed batch partitioning is also applicable for other commonsense knowledge integration models which encode target sentences and associated textual knowledge descriptions.

In this paper, we use the OK-Transformer~\cite{cui-chen-2022-enhancing} as the backbone. OK-Transformer is a recently proposed model that effectively introduces commonsense knowledge into language models. Traditional approaches for such introduction required pre-training language models on a large corpus along with external commonsense, which was time-consuming~\cite{peters2019knowledge,zhang2019ernie}. The OK-Transformer model, on the other hand, is able to directly incorporate extra knowledge without pre-training. This model utilizes commonsense tokens and attention mechanisms to effectively integrate textual commonsense. Our proposed batch partitioning method is also applicable to other models that encode target sentences and associated commonsense descriptions.

To analyze the computational complexity of encoding commonsense knowledge and formulate the problem, we briefly describe how the original OK-Transformer works. It consists of three Transformers, where $\rm{Transformer^{(1)}}$ is used to represent the target sentence, $\rm{Transformer^{(2)}}$ is used to represent each textual commonsense description, and $\rm{Transformer^{(3)}}$ is used to incorporate commonsense embeddings from $\rm{Transformer^{(2)}}$ into $\rm{Transformer^{(1)}}$.

We now concretely analyze the complexity of integrating external textual commonsense. 
%We leverage the structured knowledge integration framework~\cite{zhang2019ernie}. Similar model frameworks are also used in~\cite{peters2019knowledge}. We replace the embedding module for structured knowledge atoms in the original framework with a text encoder to encode the external textual knowledge. In this paper, we use Transformer~\cite{vaswani2017attention} as the textual knowledge encoder. The complexity analysis results and optimization methods presented in this paper can also be applied for models other than Transformer.
When encoding a sample with associated commonsense descriptions, the complexity consists of three modules:
\begin{itemize}
    \item For encoding the target sentence via $\rm{Transformer^{(1)}}$, the complexity of encoding a sentence of length $L$ into dimension $D$ is $O(L^2D)$.
    \item For encoding textual commonsense descriptions via $\rm{Transformer^{(2)}}$, the complexity of encoding $C$ knowledge descriptions of length $L$ is $O(CL^2D)$.
    \item For integrating the knowledge embeddings into the target sentence via $\rm{Transformer^{(3)}}$, the complexity is $O(C^2D)$.
\end{itemize}

\begin{table}[!htb]
\centering
\begin{tabular}{lc}
\toprule
Module                      & Complexity \\ \hline
Target sentence encoding    &    $O(L^2D)$        \\
External knowledge encoding &    $O(CL^2D)$        \\
Knowledge integration       &    $O(C^2D)$       \\
\bottomrule
\end{tabular}
\caption{Module complexities.}
\label{tab:complexity}
\end{table}

We summarize the complexity in Table~\ref{tab:complexity}. Since in practice we usually have $L^2 \gg C$, the key is is to reduce the complexity of encoding for textual commonsense descriptions, i.e., reduce $O(CL^2D)$.
%Therefore, for a sample with $c$ knowledge descriptions, its encoding complexity is $O(l^2d + cd + cm^2d)$, in which $O(cd + cm^2d)$ denotes the extra costs for integrating textual knowledge. Since in practice we usually have $m^2 \gg c$, the key is to reduce the complexity of encoding for textual knowledge, i.e., reduce $O(cm^2d)$.

{\bf Relation to retrieval-based knowledge incorporation} Integrating text commonsense is related to learning dense retrievers for efficiently retrieving and introducing external textual commonsense, such as REALM~\cite{guu2020retrieval}. In commonsense incorporation, each sample only retrieves a small number of knowledge descriptions based on trigger words. So the key of our problem is to efficiently and effectively incorporate certain knowledge descriptions, rather than the information retrieval in dense retrievers. Specifically, dense retrievers typically consist of a retriever and a knowledge-augmented encoder. Our work can be analogous to reducing the cost of the knowledge-augmented encoder.

%Meanwhile, the backbone of our method, OK-Transformer, mainly addresses how to introduce knowledge descriptions once they have been retrieved. Therefore, the challenge we are addressing is not retrieval but how to efficiently and effectively introduce relevant knowledge after retrieving them. 

%  For dense retrievers such as REALM, the role of OK-Transformer should be analogous to the knowledge-augmented encoder module.

\subsection{Problem Formulation}
\label{sec:pdef}

We now formulate the problem of batch partitioning. As stated in the introduction, different samples may correspond to the same textual commonsense description. We only need to encode the distinct commonsense descriptions once for a batch of samples. Therefore, the goal of batch partitioning is to minimize the number of distinct commonsense descriptions per batch.
%In Fig.~\ref{fig:with_batch}, since samples $x_1$ and $x_3$ are in the same batch and share the same knowledge description $t_1$, we only need to encode $t_1$ once and reuse it for the two samples. That is, 

%So for a batch with $s$ samples $x_1, \cdots, x_s$, if we use $T(x_i)$ to denote the collection of candidate textual knowledge descriptions of $x_i$, then we the number of knowledge descriptions we need to encode is $|\bigcup_{i=1}^s T(x_i)|$.

More formally, suppose the training data is $\mathcal{D}_{train}=\{x_i,T(x_i),y_i\}_{i=1}^N$, where $x_i$ is the original sample, $y_i$ is the corresponding label, and $T(x_i)=\{t_{i1},\cdots,t_{ic_i}\}$ is a collection of external knowledge descriptions for $x_i$. For a batch with $s$ samples $x_1, \cdots, x_s$, the number of knowledge descriptions we need to encode is $|\bigcup_{i=1}^s T(x_i)|$.

For convenience, we assume that $N$ is divisible by batch size $s$. To reduce the time complexity, we need to partition $\mathcal{D}_{train}$ into $k=N/s$ batches $B_1,\cdots,B_k$ such that each batch contains $s$ samples and the total number of distinct textual commonsense descriptions in each batch is minimized:
\begin{equation}
\label{eqn:prob}
    \begin{aligned}
    \min & \;\; \sum_{i=1}^k |\bigcup_{x \in B_i} T(x)| \\
    \text{s.t.} & \;\;  |B_i|=s \;\;
    \text{(size constraint for each batch)}
    \end{aligned}
\end{equation}

%% file: fast_method.tex
\section{Solving the Batch Partitioning Problem}
%The batch partitioning problem of Eq.~\eqref{eqn:prob} is NP-hard. 
To solve the batch partitioning problem, we first approximate the upper bound of Eq.~\eqref{eqn:prob} in \S~\ref{sec:upper}. We minimize its upper bound instead of directly minimizing Eq.~\eqref{eqn:prob}. In \S~\ref{sec:graph_k_cut}, we show that optimizing the upper bound can be reduced to the classic minimum graph $k$-cut problem, so that some well-studied algorithms can be applied. We show how we adapt the classical spectral clustering to this problem in \S~\ref{sec:spectral}, and how to scale it up in \S~\ref{sec:scale}.

\subsection{Upper Bound Analysis}
\label{sec:upper}
We analyze the upper bound of Eq. (1) in Theorem~\ref{theo:upper}.

\begin{theorem} [Upper bound]
\small
\setlength{\abovedisplayskip}{3pt}
\setlength{\belowdisplayskip}{3pt}
\label{theo:upper}
    \begin{equation}
    \begin{aligned}
        & \quad \sum_{i=1}^k |\bigcup_{x \in B_i} T(x)| \\ 
        & \le \sum_{i=1}^k [\sum_{x \in B_i} |T(x)| - s \E_{x_a,x_b \in B_i,x_a \neq x_b} |T(x_a) \cap T(x_b)|]
    \end{aligned}
    \end{equation}
\end{theorem}

\begin{proof}
For a batch $B$ with $s$ samples $\{x_1,\cdots,x_s\}$, we have:
\begin{equation}
\small
\setlength{\abovedisplayskip}{3pt}
\label{eqn:relaxation}
\begin{aligned}
& \quad |\bigcup_{i=1}^{s}T(x_i)| 
%&=|T(x_1)|+|T(x_2)-T(x_1)|+ \cdots+|T(x_k)-\bigcup_{i=1}^{s-1}T(x_i)| \\
=\sum_{i=1}^{s}|T(x_i)-\bigcup_{j=1}^{i-1}T(x_j)| \\
&=\sum_{i=1}^{s}|T(x_i)-\bigcup_{j=1}^{i-1}T(x_j)\cap T(x_i)| \\
&=\sum_{i=1}^{s}|T(x_i)|-\sum_{i=1}^{s}|\bigcup_{j=1}^{i-1}T(x_j)\cap T(x_i)| \\
&\le\sum_{i=1}^{s}|T(x_i)|-\sum_{i=1}^{s}{\max\limits_{1 \le j \le i-1}|T(x_j)\cap T(x_i)|} \\
\end{aligned}
\end{equation}

The upper bound in Eq.~\eqref{eqn:relaxation} after relaxation is related to the sample order of that batch, while our original objective in Eq.~\eqref{eqn:prob} is actually order-independent. To introduce order-independence, let $\pi$ be an arrangement of $1\cdots s$ that $\pi_i\in\{1,\cdots,s\}$. Noticing that $\sum_{i=1}^{s}|T(x_i)|$ is a constant, based on the order-independence, we transform Eq.~\eqref{eqn:relaxation} into the expectation under different $\pi$s:
\begin{equation}
%\small
\begin{aligned}
& \quad \E_\pi\sum_{i=1}^{s}{\max\limits_{1 \le j \le i-1}|T(x_{\pi_j})\cap T(x_{\pi_i})|} \\
& =\sum_{i=1}^{s}{\E_\pi \max\limits_{1 \le j \le i-1}|T(x_{\pi_j})\cap T(x_{\pi_i})|} \\
& \geq\sum_{i=1}^{s}\max\limits_{1 \le j \le i-1} \E_\pi|T(x_{\pi_j})\cap T(x_{\pi_i})| \\
%& =\sum_{i=1}^{s}{\max\limits_{1 \le j \le i-1}\E_{a\neq b}|T(x_a)\cap T(x_b)|} \\
%& =s \E_{1\le a,b \le s-1,a\neq b}|T(x_a)\cap T(x_b)|
& = s \E_{x_a,x_b \in B_i,x_a \neq x_b} |T(x_a) \cap T(x_b)|
\end{aligned}
\end{equation}

Therefore Theorem~\ref{theo:upper} holds.
\end{proof}

It is worth highlighting that the relaxation in the last inequality of Eq.~\eqref{eqn:relaxation} is valid due to the non-random distribution of words in samples. Specifically, samples with similar meanings tend to have similar word distributions. By grouping similar samples into the same batch, each sample pair within a batch will possess similar textual commonsense knowledge descriptions. This allows us to use the maximal common descriptions between $T(x_i)$ and $T(x_j)$ as an approximation for the common descriptions between $T(x_i)$ and $\bigcup_{j=1}^{i-1}T(x_j)$.

According to Theorem~\ref{theo:upper}, since $\sum_{i=1}^s \sum_{x \in B_i} |T(x)|=\sum_{x \in D_{train}} |T(x)|$ is a constant, minimizing Eq.~\eqref{eqn:prob} is equivalent to maximizing:
\begin{equation}
\label{eqn:max}
    \sum_{i=1}^k \E_{x_a,x_b \in B_i,x_a \neq x_b} |T(x_a) \cap T(x_b)|
\end{equation}
We will show that this is a balanced graph $k$-cut problem in \S~\ref{sec:graph_k_cut}.

\subsection{Connection to the Graph $k$-Cut Problem}
\label{sec:graph_k_cut}

We now illustrate the relationship between Eq.~\eqref{eqn:max} and the graph $k$-cut problem. We demonstrate that, with proper transformation, maximizing Eq.~\eqref{eqn:max} can be reduced to the graph $k$-cut problem. Additionally, in \S~\ref{sec:spectral}, we explain how to incorporate the constraint of the size of each mini-batch using the balanced graph $k$-cut.

Consider constructing a weighted graph $G(V,E)$ as follows:
\begin{itemize}
    \item For each sample $x_i$ in the training data, create a vertex $v_i$.
    \item For each pair of distinct vertices $(v_i,v_j)$, create an edge between them with a weight of $|T(x_i) \cap T(x_j)|$.
\end{itemize}

The graph $k$-cut for $G(V,E)$ partitions $G(V,E)$ into $k$ non-empty components: $V_1,\cdots,V_k$ such that the sum weight of cross-component edges is minimized. 
According to the construction of $G(V,E)$, %the sum weight of inner-component edges of the $k$-cut is equal to Eq.~\eqref{eqn:max}. Therefore, 
maximizing Eq.~\eqref{eqn:max} is equivalent to minimizing the sum weight of the cut. This is formalized in Theorem~\ref{theo:relation}.

\begin{theorem}[Relation to minimum $k$-cut problem]
\label{theo:relation}
Suppose the weight of the $k$-cut for $G(V,E)$ is $w$, then we have:
\begin{equation}
\label{eqn:upper}
    Eq.~\eqref{eqn:max} = \frac{2}{s(s-1)} \sum_{i=1}^{n-1} \sum_{j=i}^{n} |T(x_i)\cap T(x_j)| - w
\end{equation}
\end{theorem}
\begin{proof}
A k-cut of $G(V,E)$ consists of $k$ components. These $k$ components correspond to $k$ batches in the $k$-partition. Therefore, the sum weight of inner-component edges of the $k$-cut is equal to Eq.~\eqref{eqn:max} $* \frac{s(s-1)}{2}$. Since the total weight of edges in $G(V,E)$ is equal to the sum weight of inner-component edges plus the sum weight of the cut, Theorem~\ref{theo:relation} holds.
\end{proof}
As $\sum_{i=1}^{n-1} \sum_{j=i}^{n} |T(x_i)\cap T(x_j)|$ is a constant for the given training data, Theorem~\ref{theo:relation} shows that maximizing Eq.~\eqref{eqn:max} is equivalent to minimizing the $k$-cut for $G(V,E)$. Thus, we convert the problem of maximizing Eq.~\eqref{eqn:max} into the classic minimum $k$-cut problem.

\subsection{Spectral Clustering for the Balanced $k$-Cut}
\label{sec:spectral}

Based on the analysis in \S~\ref{sec:graph_k_cut}, we propose to use spectral clustering, a widely used approach for solving the minimum graph $k$-cut problem, as our batch partition algorithm. Spectral clustering employs spectral relaxation of the ratio/normalized cut and uses k-means in the embedding of the vertices found by the first $k$ eigenvectors of the graph Laplacian in order to obtain the clustering. In addition to the classic minimum graph $k$-cut problem, we need to incorporate the constraint that each cut/batch must have a size of $s$.

To incorporate the batch size constraint, we make a simple modification to the k-means step in spectral clustering. In the traditional k-means, each node is assigned to the nearest cluster center. In our algorithm, if the nearest cluster center has already been assigned $s$ nodes, the node will be assigned to the nearest center that has fewer than $s$ assigned nodes. The specific spectral clustering algorithm is presented as follows.

\begin{enumerate}[itemindent=0.0em]
%\item Construct $G(V,E)$ according to \S~\ref{sec:graph_k_cut}.
\item Compute the spectral embedding $Y \in \mathbb{R}^{n\times k}$ by stacking the normalized first $k$ eigenvectors of $G(V,E)$ in columns as described in~\cite{ng2002spectral}.
\item Treat the $i$-th row of $Y$ as the feature of the $i$-th training point $e_i \in \mathbb{R}^k$.
\item Given an initial set of $k$ means $m_1,\cdots,m_k$ by randomly selecting $k$ nodes as centers, repeat the following two steps until convergence:
\begin{enumerate}[itemindent=0em]
\item {\bf Assignment step} Assign nodes to centers:
\begin{enumerate}[itemindent=0em]
\item Compute distances to centers $dis_{i,j}=distance(e_i,m_j)$, where the Euclidean distance is used.
\item Sort $i,j$ in ascending order of $dis_{i,j}$ for all $1\le i \le n,1 \le j\le k$.
\item Iterate through all $i,j$. If node $i$ is not assigned in this round and center $j$ has less than $s$ assigned nodes, assign node $i$ to center $j$.
\end{enumerate}
\item {\bf Update step} Compute new centers by taking the mean of their assigned nodes.
\end{enumerate}
\end{enumerate}

\subsection{Spectral Clustering at Scale} 
\label{sec:scale}
The above algorithm consists of computation of the eigenvectors, and the use of k-means. K-means is efficient even for large-scale data. However, when $n$ and $k$ are large, the graph construction and eigenvectors computation become computationally expensive. %To speed up these two modules, we use the following techniques

%{\bf Efficient eigenvectors computation} 
To compute the spectral embeddings at scale, high-performance optimization techniques are available such as~\cite{liu2013large,kolev2016note,boutsidis2015spectral,tremblay2016compressive}. Also, in our experiments, a simple trick was found that yields meaningful results: only calculate $k'$-dimensional feature vectors ($k'<k)$ and perform k-means with the $k'$ dimensions. We found that $k'=8$ is a good practice in our experiments.

%{\bf Efficient graph construction} To efficiently construct $G(V,E)$, we take advantage of its sparsity. Since external textual knowledge is organized by trigger words, $E_{i,j}=|T(x_i)\cap T(x_j)|>0$ only if $x_i$ and $x_j$ have common trigger words. Therefore, we use a sparse affinity matrix to represent $G(V,E)$. To compute the neighbors of $i$, we iterate over the trigger words in $x_i$ and add all $x_j$ that also contain the same trigger words. This can be easily implemented using a key/value store for trigger words/samples.

%% file: fast_related.tex
\section{Related Work}
{\bf Integrating knowledge into language models} has been one of the focuses of language modeling research in recent years. The main integration methods currently include using pre-trained entity embeddings, and constructing knowledge-aware corpora. ERNIE~\cite{zhang2019ernie}, KnowBERT~\cite{peters2019knowledge}, and KGLM~\cite{logan2019barack} are typical methods using pre-trained entity embeddings. ERNIE uses Wikidata~\cite{vrandevcic2014wikidata} as the knowledge base and uses TransE~\cite{bordes2013translating} to encode knowledge. KnowBERT, on the other hand, uses skip-gram like objective~\cite{mikolov2013distributed} based on Wikipedia descriptions as the pre-trained entity embeddings. In addition, KnowBERT adds a loss on entity linking to the pre-trained objective. KGLM~\cite{logan2019barack} allows modification/updating of knowledge by building a local knowledge graph for the target sentence. WKLM~\cite{xiong2019pretrained} constructs a corpus of incorrect knowledge descriptions by replacing Wikipedia's entities with different entities of the same type. It trains the model to identify incorrect and correct knowledge descriptions. Recently, models that integrate textual knowledge have also been proposed.
In this paper, we adopt the model structure in OK-Transformer~\cite{cui-chen-2022-enhancing}.

{\bf Textual knowledge bases} Noting the deficiencies of symbolic knowledge in terms of expressiveness and contextual information representation, some work has started to use text as a form of knowledge. ATOMIC~\cite{sap2019atomic,hwang2020comet} is a large-scale manually annotated common-sense textual knowledge base that includes social interaction, event-centered, physical entity. ATOMIC contains knowledge like {\it (PersonX reaches PersonX's home, Before, PersonX needs to park the car)}. ASER~\cite{zhang2020aser} is an eventuality knowledge graph of activities, states, events, and their relations. Its knowledge atoms are in natural language form, e.g. {\it (I do not have lunch, succession, I am hungry)}. COMET~\cite{bosselut2019comet} is an extension of ATOMIC based on the generative language model. It mainly solves the problem of insufficient coverage of ATOMIC. Some primitive research~\cite{guan2020knowledge,shwartz2020unsupervised} has started to apply these textual knowledge bases in some specific tasks. OK-Transformer~\cite{cui-chen-2022-enhancing} is proposed to integrate textual knowledge for general purposes. However, in our experimental tests, it takes too much time in encoding the commonsense. To our knowledge, there is still a lack of research on how to integrate textual knowledge into general text understanding tasks efficiently.

{\bf Comparison with dense textual knowledge retriever} When introducing external texts, another style is to use a retriever that returns only top $k$ candidate texts in terms of similarity~\cite{chen2017reading,karpukhin2020dense,wang2019multi}. However, this method requires a heavy pre-training process to learn the retriever. On the other hand, for the textual knowledge base we use in this paper, we can directly use the manually labeled trigger words for each knowledge description to retrieve knowledge. Therefore, in this paper, we focus on how to efficiently and effectively integrate knowledge from a textual knowledge base.%, rather than the knowledge retrieval in a dense textual knowledge retriever.

{\bf High-performance language models} More general techniques for high-performance language models have also received extensive studies. The main approaches of previous studies include (1) model compression and quantization~\cite{sanh2019distilbert,jacob2018quantization}, and (2) efficient representation of long texts~\cite{kitaev2019reformer,peng2020random}. However, the model compression approaches require heavy pre-training before they can be adapted to language models. Moreover, the techniques for optimizing the efficiency for long text do not have significant effects on short texts~\cite{peng2020random}. Besides, each commonsense description we considered in this paper tends to be short. 
%Therefore previous work on high-performance language models has a limited effect on textual knowledge integration. 
In addition, these works have not considered the characteristics of the knowledge integration problem in this paper, i.e., a training sample corresponds to multiple candidate textual knowledge from the knowledge base.

%% file: fast_exp.tex
\begin{table*}[!htb]
\centering
\small
\setlength{\tabcolsep}{3pt}
\begin{tabular}{lccccccccc}
\toprule
\multicolumn{1}{l}{}                           & LM  & Comm.QA                 & PhysicalQA                    & WSC273                       & PDP                        & WinoGrande                    & WinoGender  &   Avg. & Speed-up $\uparrow$                    \\ \hline
BERT                                           & BERT & 55.86 & 68.71 & 66.30 & 85.00 & 51.38 & 68.19 & 65.44 & - \\
Frozen knowledge                       &BERT & 56.43                         & 68.06                         & 65.93                        & 83.33             & 51.30                          & 68.47     & 65.59 &    $\bf 1.4 \times$               \\ 
OK-Transformer                          & BERT & 56.27               & 69.09                & \textbf{67.40}               & \textbf{86.67}             & \textbf{52.64}                & 71.53    &   66.56 & $1.0 \times$        \\
\textbf{Batch Partitioning}                                  &BERT & \textbf{56.59}                & \textbf{69.53}                & 66.67               &  \textbf{86.67}             & 52.17                & \textbf{72.78}  & \textbf{67.40} &      $\bf 1.4 \times$       \\ \midrule
RoBERTa                                        &RoB. & 73.55                         & 79.76                         & 90.10                         & 90.00                         & -                             & 94.60      & 83.95 &  -                  \\
Frozen knowledge & RoB. & 75.02                      &  52.77             & 90.48             & 88.33                      & -                             & \textbf{96.81}    &  80.01 &     $\bf 1.5 \times$             \\
OK-Transformer                            & RoB. & \textbf{75.92}                & 80.09                & \textbf{91.58}               & 90.00                & -                             & 95.00      &      84.75 & $1.0 \times$      \\
\textbf{Batch Partitioning}                           & RoB. & 75.59                & \textbf{80.20}                          & 90.48               & \textbf{91.66}             & -                             & 96.25    &    \textbf{85.14} &  $1.4 \times$      \\
\bottomrule
\end{tabular}
\caption{Results on commonsense reasoning tasks. The effectiveness of batch partitioning surpasses the vanilla BERT/RoBERTa, and is competitive with its upper bound (OK-Transformer). In terms of efficiency, the speed-up of batch partitioning is also competitive to its upper bound (frozen knowledge). RoB. denotes RoBERTa.}
\label{tab:csr}
\end{table*}

\begin{table*}[htb]
\centering
\small
\setlength{\tabcolsep}{6pt}
\begin{tabular}{lccccccccc}
\toprule
\multicolumn{1}{l}{}  & LM                        & MRPC                 & CoLA           & RTE            & QNLI           & STS-B       & SST-2  &   Avg. & Speed-up        \\ \hline
BERT & BERT                                          & 86.52/90.66          & 59.50           & 71.43          & 91.20           & 89.35/88.93 & 91.97   & 82.28 &  -       \\
Frozen knowledge & BERT                       & 87.50/91.28           & 57.31          & 70.76          & 91.71          & 87.31/87.20 & 92.43  & 81.78 &
$\bf 2.3 \times$ \\
OK-Transformer & BERT                              &87.50/91.04  & 58.29          & \textbf{72.20}  & \textbf{91.58} & \textbf{89.82/89.46} & 92.66 & 82.54 & $1.0 \times$ \\
\textbf{Batch Partitioning}   & BERT                               & \textbf{87.99/91.45} & \textbf{61.41} & 71.48          & 91.32 & 89.64/89.19 & \textbf{93.69} & \textbf{83.09} & $2.1 \times$ \\ \midrule
RoBERTa & RoB.                                        & 90.49/93.07          & 66.84          & 86.28          & 93.37          & 91.83/91.95 & 95.64    &  87.86 & -      \\
Frozen knowledge  & RoB. & 89.71/92.61          & \textbf{68.22} & \textbf{87.36} & 94.39          & 90.74/90.47 & 96.10 & 88.19 & $\bf 2.4 \times$ \\ 
OK-Transformer  & RoB.                           & \textbf{91.91/94.24} & 66.89          & 86.28          & \textbf{94.71} & 92.19/92.36 & \textbf{96.44}  & \textbf{88.49} & $1.0 \times$ \\
\textbf{Batch Partitioning}  & RoB.                             & 90.69/93.44 & 67.75          & 85.92          & 94.07 & \textbf{92.41/92.20} & 96.22 & 88.27 & $2.1 \times$ \\
\bottomrule
\end{tabular}
\caption{Results on text classification tasks. Both the effectiveness and the efficiency of batch partitioning are competitive to their upper bounds (OK-Transformer and frozen knowledge). }%In terms of accuries, the OK-BERT/RoBERTa (our) achieve competitive performance with OK-BERT/RoBERTa (random). In terms of efficiency, OK-BERT/RoBERTa (our) achieve 100\% speed up on average when compared to OK-BERT/RoBERTa (random), and are competitive with OK-BERT/RoBERTa-static (random).}
\label{tab:glue}
\end{table*}

\section{Experiments}
In this section, we conducted extensive experiments to evaluate batch partitioning. We aim to address the following key questions:
\begin{enumerate}
    \item (\S~\ref{sec:exp:effect}) How much is the efficiency improvement of batch partitioning? Can it improve efficiency without sacrificing effectiveness?
    \item (\S~\ref{sec:exp:scalability}) What is the scalability of batch partitioning as an acceleration method, and can it be applied to large-scale training?
    \item (\S~\ref{sec:exp:rationale}) Is the main theoretical contribution of this paper, i.e., solving the balanced graph-$k$ cut by spectral clustering, consistent with the real datasets?
\end{enumerate}

\subsection{Implementation Details and Setup}

{\bf Textual knowledge base} We follow~\citep{cui-chen-2022-enhancing} to use ATOMIC2020~\cite{hwang2020comet} as the textual knowledge base. Each atom in ATOMIC2020 is commonsense in text form. For each sentence in the downstream task, we retrieve the knowledge associated with it from the textual knowledge base. Note that, unlike retrieving knowledge from free text~\cite{guu2020retrieval}, the textual knowledge base ATOMIC2020 is constructed manually, and each knowledge description has corresponding trigger words. These trigger words are usually verbs or verb phrases. We retrieve related textual commonsense descriptions by keyword-matching of these trigger words.

{\bf Model architecture} We use OK-Transformer~\cite{cui-chen-2022-enhancing} as the backbone of our model. It directly incorporates extra knowledge without pre-training.
%uses Transformer to represent the original sentence, and integrate external knowledge via attention. 
OK-Transformer is based on either BERT or RoBERTa. We use OK-Transformer based on BERT by default. We also follow the hyperparameter settings of OK-Transformer. All experiments were run on 8 Nvidia RTX 3090Ti GPUs.

{\bf Datasets} We evaluate batch partitioning via commonsense reasoning and sentence classification. Since the textual knowledge introduced in this paper is commonsense descriptions, we first verify whether the proposed method in this paper could be applied to the commonsense reasoning tasks. To this end, we choose a wide range of commonsense reasoning tasks to conduct the experiments: CommonsenseQA \cite{talmor2019commonsenseqa}, PhysicalQA \cite{bisk2020piqa}, as well as several Winograd Schema Challenge (WSC) datasets including WSC273~\cite{levesque2012winograd}, PDP~\cite{morgenstern2016planning}, WinoGrande~\cite{sakaguchi2019winogrande}, WinoGender~\cite{rudinger2018gender}. Furthermore, for a comprehensive comparison, we also evaluate the efficiency and effectiveness of the proposed batch partitioning method on the text classification benchmark GLUE~\cite{wang2018glue}.

\subsection{Effectiveness and Efficiency}
\label{sec:exp:effect}

{\bf Baselines} To verify the efficiency and effectiveness of batch partitioning, we used the following baselines:
\begin{itemize}
    \item {\bf Vanilla BERT/RoBERTa} without external knowledge. %To illustrate the improvement of the integrated textual knowledge on the model effectiveness, we compared our method with the vanilla BERT/RoBERTa models, which introduce no extra knowledge.
    \item {\bf OK-Transformer} To show the efficiency gains of the batch partitioning proposed in this paper, we compare it with the original OK-Transformer. The baseline randomly partitions samples into batches. We consider this baseline as {\bf the lower upper bound of effectiveness} of commonsense integration.
    \item {\bf Frozen knowledge encodings} For a comprehensive comparison, we propose to freeze the encoding of commonsense descriptions during fine-tuning. This approach allows us to introduce external textual commonsense descriptions via embedding lookup with minimal time cost. We consider this baseline as {\bf the upper bound on the efficiency} of commonsense integration.
\end{itemize}

The results of commonsense reasoning and text classification are presented in Table~\ref{tab:csr} and Table~\ref{tab:glue}, respectively. The effectiveness of our batch partitioning approach is demonstrated by its improvement over vanilla language models on both commonsense reasoning and text classification tasks. The effectiveness is comparable or slightly superior to that of OK-Transformer, which serves as the upper bound for effectiveness. In terms of efficiency, our approach significantly accelerates knowledge integration models across a range of tasks. On average, it reduces the time cost for knowledge encoding by 40\% for commonsense reasoning tasks, and 110\% for text classification tasks. This acceleration is close to the frozen knowledge, and serves as the upper bound for efficiency. Overall, our approach is close to its efficiency upper bound without losing effectiveness.

\begin{figure}[tb]
	\centering
		\includegraphics[scale=.4]{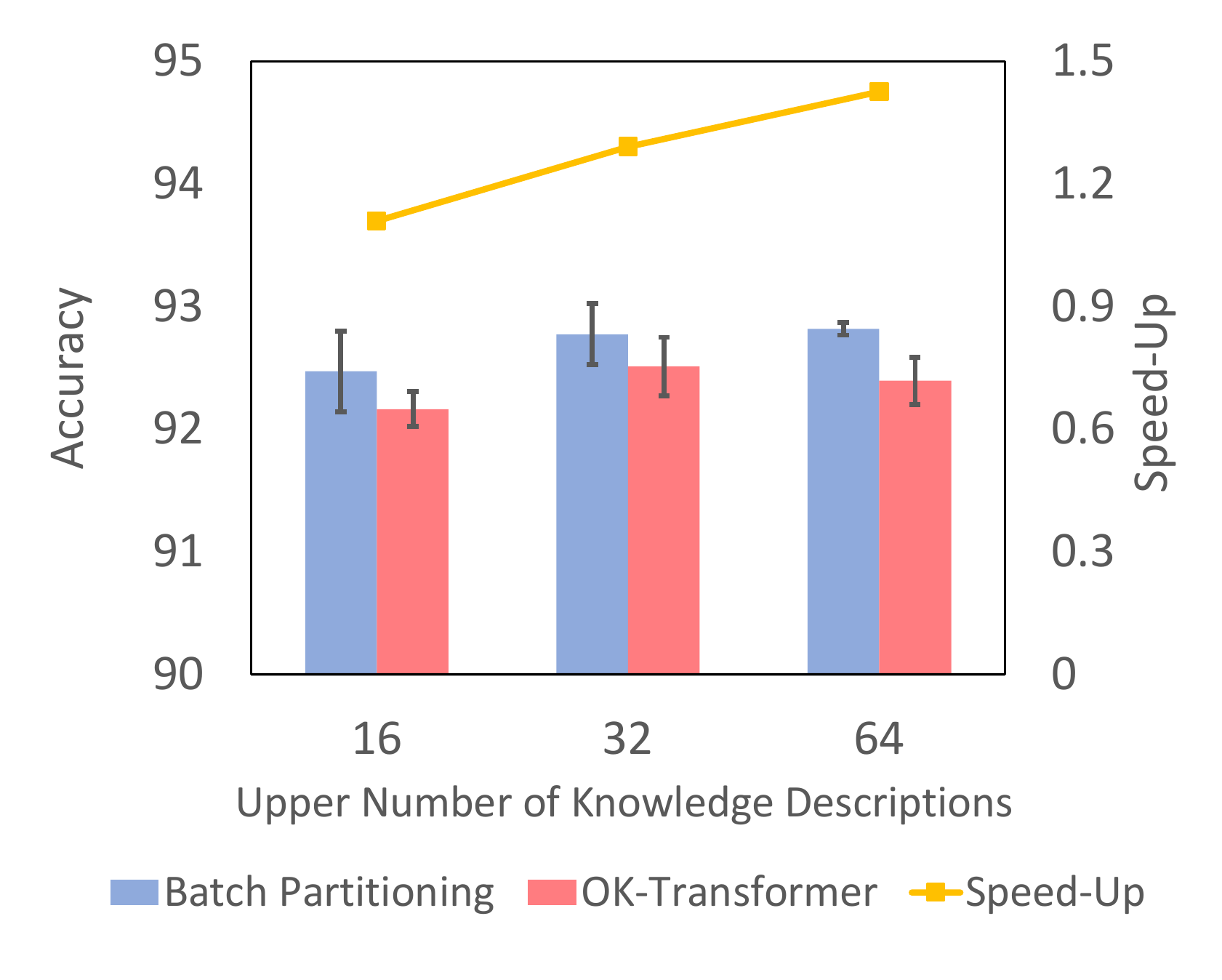}
\caption{The effect of the scale of extra commonsense. We control the scale by limiting the upper number of commonsense descriptions per sample in SST-2.}
\label{fig:commonsense_size}
\end{figure}

\subsection{Scalability for Dataset Sizes, Device Capacities, and Knowledge Sizes}
\label{sec:exp:scalability}

\begin{figure*}[!tb]
\centering
\subfloat[MRPC]{\includegraphics[width=0.25\linewidth]{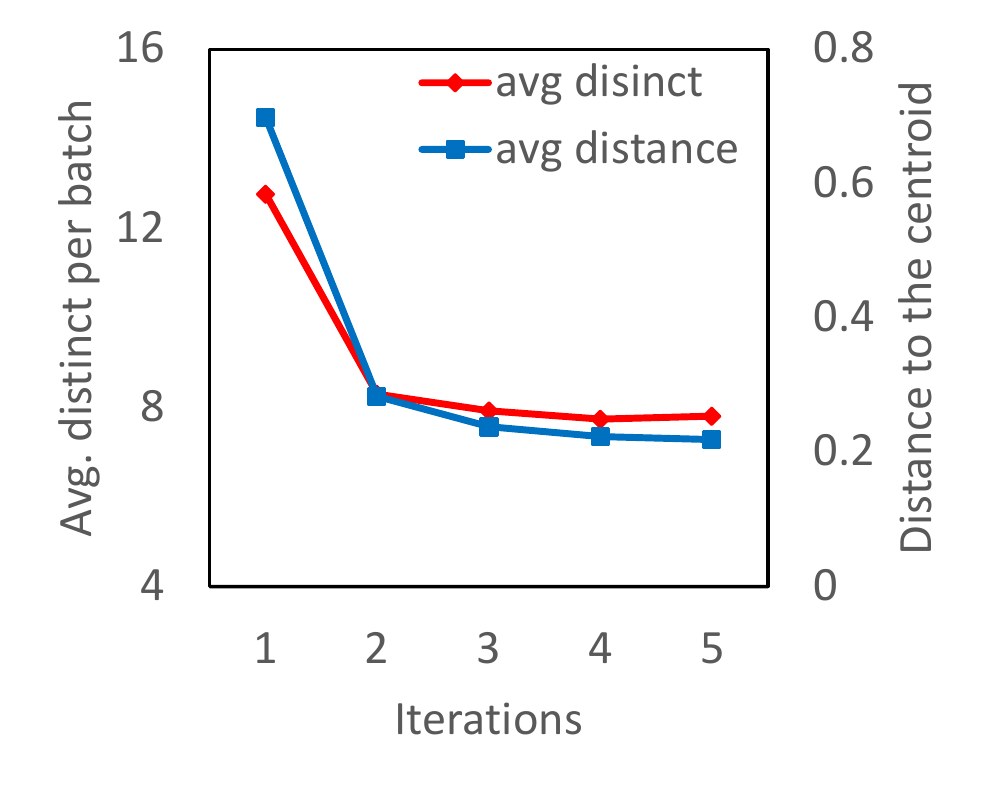}}
\hfill
\subfloat[CoLA]{\includegraphics[width=0.25\linewidth]{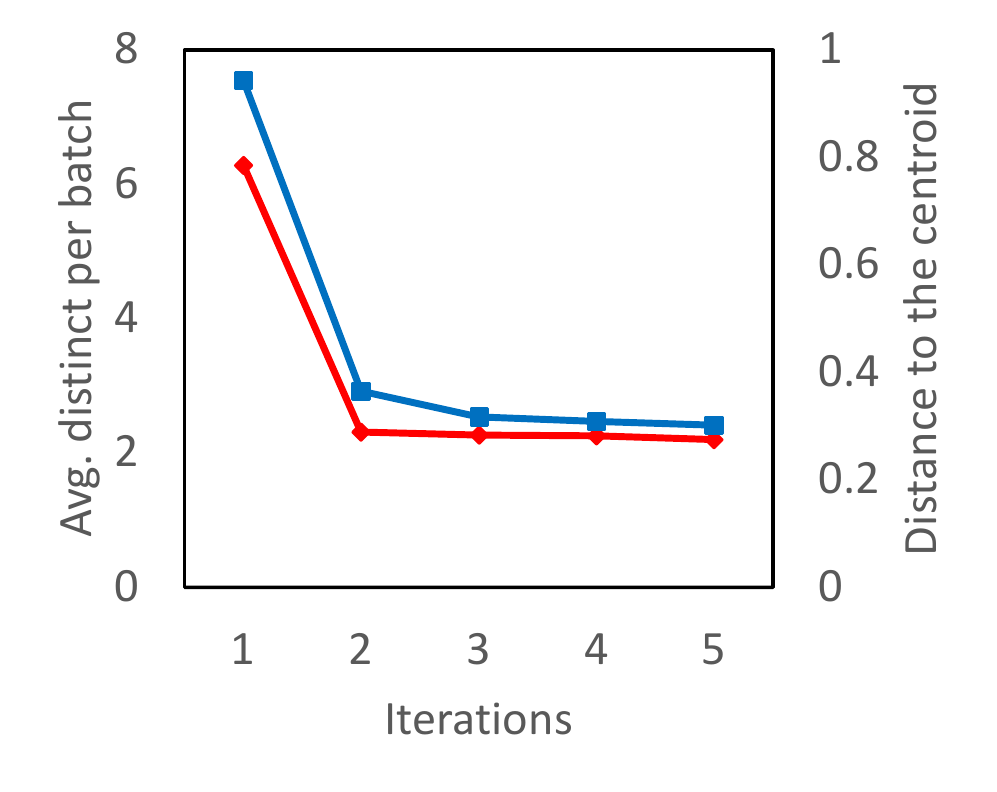}}
\hfill
\subfloat[RTE]{\includegraphics[width=0.25\linewidth]{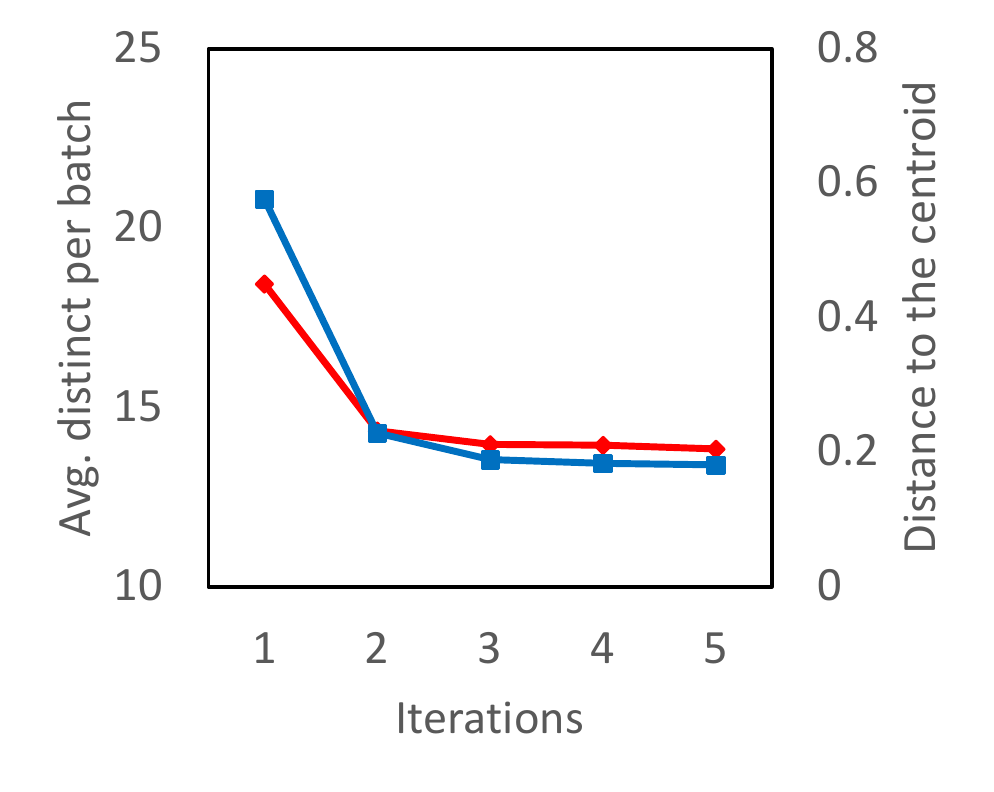}}
\hfill
\subfloat[STS-B]{\includegraphics[width=0.25\linewidth]{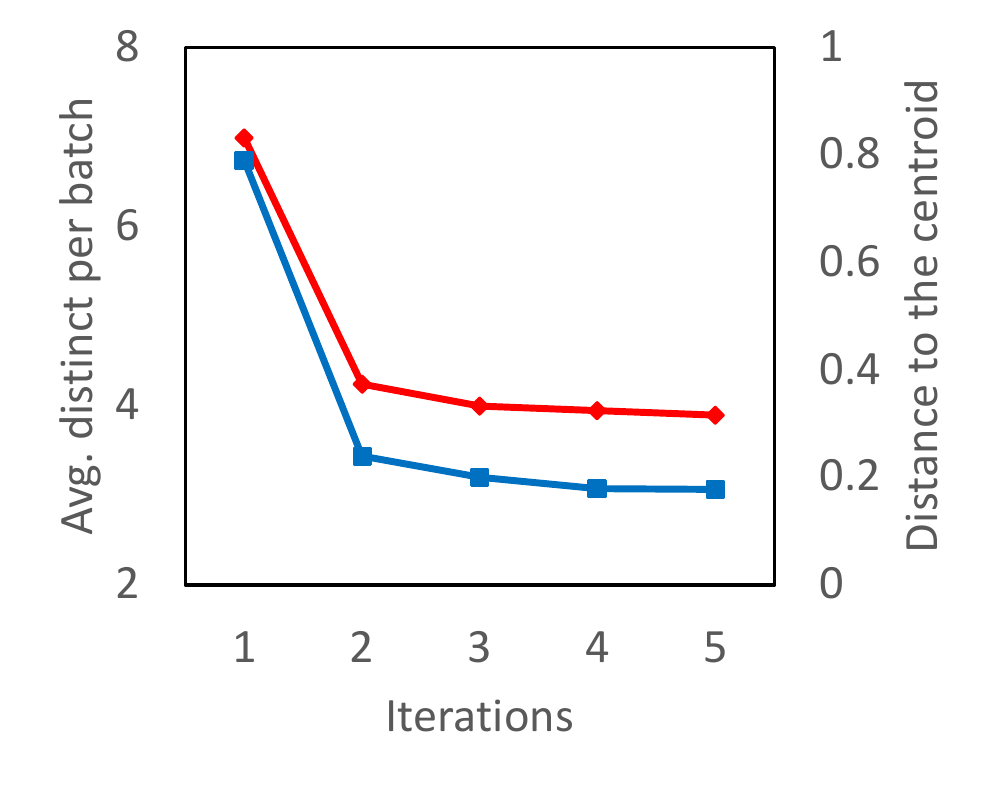}}
\caption{Strong correlation between the distinct commonsense descriptions per batch and the average distance to the centroid during clustering.}
\label{fig:rationale}
\end{figure*}

\begin{figure}[tb]
	\centering
		\includegraphics[scale=.5]{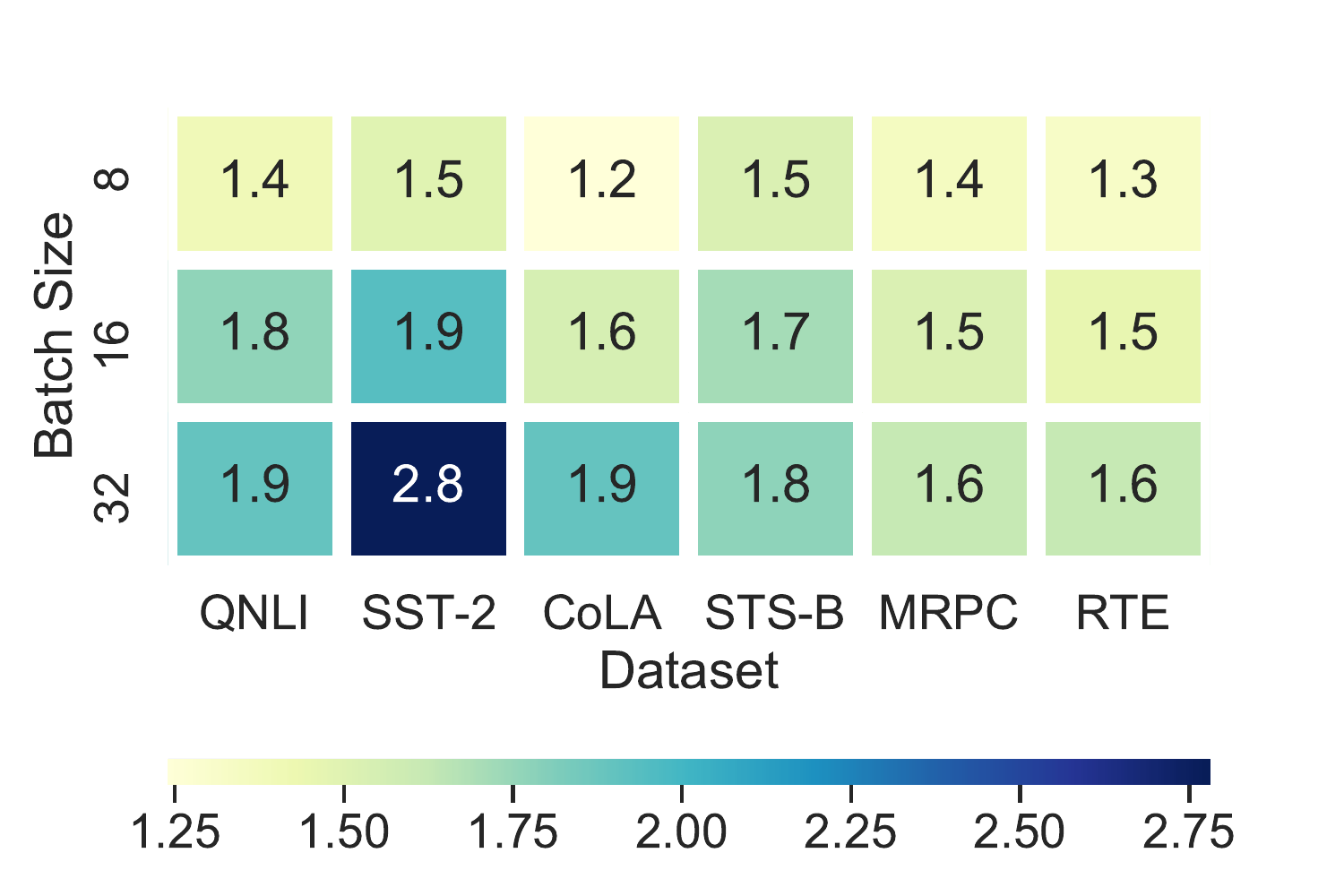}
\caption{Speed-up of proposed batch partitioning method over different training batch size and dataset size. Note that the datasets are arranged in descending order according to dataset size.}
\label{fig:batch_size}
\end{figure}
% \begin{table}[ht]
% \centering
% \small
% %\setlength{\tabcolsep}{2pt}
% \caption{Speed up of proposed batch partitioning method over different training batch size and dataset size.}
% \begin{tabular}{c|cccccc}
% \toprule
%                      & QNLI    & SST-2   & CoLA    & STS-B   & MRPC    & RTE     \\ \midrule
% Dataset Size         & 110206  & 68221   & 9594    & 7249    & 4076    & 2767    \\ \midrule
% Batch Size=8         & $1.40 \times$ & $1.50 \times$ & $1.24 \times$ & $1.53 \times$ & $1.36 \times$ & $1.33 \times$ \\
% Batch Size=16        & $1.77 \times$ & $1.94 \times$ & $1.55 \times$ & $1.72 \times$ & $1.53 \times$ & $1.46 \times$ \\
% Batch Size=32        & $1.90 \times$ & $2.78 \times$ & $1.90 \times$ & $1.78 \times$ & $1.63 \times$ & $1.63 \times$ \\
% \bottomrule
% \end{tabular}
% \label{tab:batch_size}
% \end{table}
% \begin{table*}[ht]
% % \small
% \centering
% \setlength{\tabcolsep}{2pt}
% \caption{Percentage of the reduced number of knowledge encodings via batch partitioning under different batch sizes.}
% \begin{tabular}{l|rrrrrr}
% \toprule
%                      & QNLI    & SST-2   & COLA    & STS-B   & MRPC    & RTE     \\ \midrule
% Dataset Size         & 110206  & 68221   & 9594    & 7249    & 4076    & 2767    \\ \midrule
% batch size=8         & 43.86\% & 73.87\% & 41.60\% & 37.56\% & 28.64\% & 20.89\% \\
% batch Size=16        & 51.83\% & 78.90\% & 45.61\% & 37.58\% & 29.85\% & 22.41\% \\
% batch Size=32        & 49.19\% & 81.66\% & 46.61\% & 36.95\% & 31.55\% & 24.95\% \\
% \bottomrule
% \end{tabular}
% \label{tab:batch_size}
% \end{table*}
In this subsection, we investigate the scalability of batch partitioning with different batch sizes, as well as the different dataset sizes. Larger dataset sizes usually mean devices with larger memory.
In particular, we calculated the speedups of knowledge encoding for different batch sizes and different tasks. The results are shown in Fig.~\ref{fig:batch_size}. The datasets are sorted by size in descending order. 

It can be clearly seen that as the size of the dataset rises or the memory of the device rises (larger batch size), the speedup of batch partitioning becomes more significant. This is because, for data-intensive tasks, the knowledge overlapping among different samples is more significant, which increases the feasibility of using batch partitioning. This result verifies the scalability of batch partitioning.
We also investigate the scalability of batch partitioning over different scales of integrated commonsense. To control the scale, we set the upper number of commonsense descriptions for each sample to 16/32/64, respectively, and study the efficiency. Intuitively, richer commonsense descriptions lead to higher effectiveness but more computation cost. The results are shown in Fig.~\ref{fig:commonsense_size}.

As commonsense knowledge becomes richer, the effectiveness and the acceleration both increase. This is because the knowledge overlapping among samples also becomes more significant.
The result verifies that batch partitioning is applicable for incorporating large-scale commonsense knowledge bases.

% the computation cost during training raise under \textcolor{red}{Random Batch Partitioning} %random
% strategy, which meets our intuition. On the other hand, with our batch partitioning method, although we raise the upper number, the time cost during training is relatively stable. We consider this is because most of the candidate knowledge has been encoded even when the upper number = 16. In the case of random partition, a larger upper limit value is required to include all the knowledge. This further reflects that our method can introduce more knowledge efficiently.

\subsection{Effect of Spectral Clustering Theory} 
\label{sec:exp:rationale}

In this paper, we propose the use of spectral clustering to solve the batch partitioning problem. We approximate and optimize the distinct number of descriptions per batch in Eq.~\eqref{eqn:prob} by minimizing the distance of each node to the centroid of the cluster in spectral clustering. In this subsection, we demonstrate the rationale behind this approximation by highlighting the strong correlation between the objective of Eq.~\eqref{eqn:prob} and the distance minimization in spectral embeddings.

To this end, we plot how the centroid distance and the distinct descriptions per batch vary at each iteration of the spectral clustering algorithm in Fig.~\ref{fig:rationale}. The results show a strong correlation between the value we directly optimize (i.e., the centroid distance) and the target of the batch partitioning (i.e., distinct descriptions per batch). This supports the feasibility of using spectral clustering to convert the batch partitioning problem into a balanced graph $k$-cut problem and solve it efficiently.

%% file: fast_conclu.tex
\section{Conclusion}

In this paper, we study how to improve the efficiency of incorporating commonsense knowledge in language models. Due to the high encoding costs of commonsense descriptions, it is crucial to reduce their encoding complexity. Our idea is that by carefully dividing samples with similar descriptions into the same batch, the knowledge encoding utilization can be improved. %That is, the encoding of one knowledge descriptions can be reused for different samples in a batch.

With such an idea, we theoretically analyze the optimization objective of this batch partitioning. We found that the upper bound of this problem can be reduced to the classical graph $k$-cut problem. We propose to use the well-studied spectral clustering algorithm to optimize the batch partitioning. By experimenting with a variety of tasks, we show that the proposed batch partitioning approaches its upper bound in terms of both effectiveness and efficiency. And the method is more applicable for larger datasets and on devices with more capabilities.

\section{Limitations}
The theoretical results and the algorithm should be applicable for other knowledge integration models which encode target sentences and associated textual knowledge descriptions in mini-batches. However, this paper does not extensively apply the proposed method to various knowledge integration models to explore its efficiency and effectiveness. 

%while we theoretically prove that the proposed method is able to improve efficiency in a wide range of knowledge integration approaches (KnowBERT, ERNIE) since it reduces the computational cost of encoding distinct knowledge. In addition, since the proposed batch partitioning method improves the efficiency of knowledge integration by reducing distinct knowledge, the improvement brought by the proposed method varies with the encoding cost of knowledge.}

% This paper lacks a formalized analysis of the relationship among virtual explanations/pre-training/model generalization. Although we try to analogize pre-training and prompt in Sec~\ref{sec:method:rationale} to explain how the virtual explanation works, it lacks a rigorous mathematical description.

% The validation of the virtual explanation is limited to relation extraction. Although we show its potential on other applications in Sec~\ref{sec:exp:beyong}, the experiments are still primitive. A more systematic evaluation on different NLP tasks is still excepted.